\documentclass[letterpaper, 10 pt, conference]{ieeeconf}  

\IEEEoverridecommandlockouts                              
\title{\LARGE \bf
Exact Formulas for Finite-Time Estimation Errors of Decentralized Temporal Difference Learning with Linear Function Approximation
}

\author{Xingang Guo and Bin Hu
\thanks{This work is generously supported by the NSF award
CAREER-2048168 and the 2020 Amazon research award.}
\thanks{Xingang Guo and Bin Hu are with the Coordinated Science Laboratory (CSL) and the Department of Electrical and Computer Engineering, 
        University of Illinois at Urbana-Champaign. Email:
        {\tt\small \{xingang2,~binhu7\}@illinois.edu}}
}

\usepackage{cite}
\usepackage{amsmath,amssymb,amsfonts,bbm}
\usepackage{algorithmic}
\usepackage{graphicx}
\usepackage{textcomp}
\usepackage{xcolor}
\usepackage[ruled,vlined]{algorithm2e}
\usepackage{accents}
\pdfoutput=1
\usepackage{hyperref}
\hypersetup{
    colorlinks=true,
    linkcolor=black,
    citecolor=black,
    filecolor=black,
    urlcolor=black,
}
\usepackage{times}
\usepackage{bm}
\usepackage{bbm}
\usepackage{multicol}

\newtheorem{theorem}{Theorem}
\newtheorem{corollary}{Corollary}

\newtheorem{remark}{Remark}

\newtheorem{assumption}{Assumption}
\newtheorem{fact}{Fact}

\newcommand{\norm}[1]{\left\|#1\right\|}

\newcommand{\field}[1]{\mathbb{#1}}
\newcommand{\R}{\field{R}}

\newcommand{\vQ}{\hat{\textbf{Q}}}
\newcommand{\bmtx}{\begin{bmatrix}}
\newcommand{\emtx}{\end{bmatrix}}
\newcommand{\bsmtx}{\left[ \begin{smallmatrix}} 
\newcommand{\esmtx}{\end{smallmatrix} \right]} 
\newcommand{\bmatarray}[1]{\left[\begin{array}{#1}}
\newcommand{\ematarray}{\end{array}\right]}

\newcommand{\N}{\mathcal{N}}
\newcommand{\M}{\mathcal{M}}
\newcommand{\E}{\mathcal{E}}
\newcommand{\G}{\mathcal{G}}

\newcommand{\tp}{\mathsf{T}}
\newcommand{\bmat}[1]{\begin{bmatrix}#1\end{bmatrix}}

\DeclareMathOperator{\vect}{vec}
\DeclareMathOperator{\diag}{diag}
\DeclareMathOperator{\sym}{sym}
\DeclareMathOperator{\real}{real}

\begin{document}

\maketitle
\thispagestyle{empty}
\pagestyle{empty}

\begin{abstract}
In this paper, we consider the  policy evaluation problem in multi-agent reinforcement learning (MARL) and derive exact closed-form formulas for the finite-time mean-squared estimation errors of  decentralized temporal difference (TD) learning with linear function approximation.
Our analysis hinges upon the fact that the decentralized TD learning method can be viewed 
as a Markov jump linear system (MJLS). Then standard MJLS theory can be applied to quantify the mean and covariance matrix of the estimation error of the decentralized TD method at every time step. Various implications of our exact formulas on the algorithm performance are also discussed. 
An interesting finding is that under a necessary and sufficient stability condition, the mean-squared TD estimation error will converge to an exact limit at a specific exponential rate.
\end{abstract}

\section{INTRODUCTION}


Reinforcement Learning (RL) provides a general paradigm for solving sequential decision making tasks, and has received much research attention in recent years \cite{Puterman2014, sutton2018reinforcement,bertsekas1996neuro}. 
An important task in RL is the policy evaluation, which aims to estimate the value function for any given policy.  Temporal difference (TD) learning combined with various function approximators has been widely used for model-free policy evaluation \cite{sutton1988learning, dann2014policy}.
The asymptotic behaviors of TD learning are well understood via applying the ordinary differential equation (ODE) method\cite{ borkar2009stochastic,TsiRoy1997,borkar2000ode}. 
 Recently, there has been a growing interest in  finite-time analysis of TD learning with linear function approximation in various settings \cite{dalal2018finite,bhandari2018finite,srikant2019finite,hu2019characterizing,xu2019two}.


In this work, we focus on the multi-agent reinforcement learning (MARL) setting \cite{zhang2021multi}, and study the finite-time behaviors of decentralized TD learning\cite{Mathkar2016}.
To perform multi-agent policy evaluation, 
a group of agents will cooperate 
to learn the global value function via exchanging local information over a communication network.
Specifically,
each agent can observe the global state of the shared environment, and execute control actions based on a local policy. Then each agent will receive local rewards, and
collaborate over the network to evaluate the global value function.
The idea of decentralized TD learning is that the agents can share their local TD estimates with neighbors and then reach a consensus for a good estimate for the global value function. 

The asymptotic convergence of decentralized TD learning is well understood \cite{Mathkar2016}. More recently, several upper bounds for the finite-time mean-squared estimation errors of decentralized TD learning have been obtained under a variety of assumptions \cite{Doan2019, Sun2020, Doan2021,Zeng2021,Wang2020}. Specifically, the IID noise case was covered in \cite{Doan2019}, and the more general Markov noise case has been addressed in \cite{Sun2020, Doan2021,Zeng2021,Wang2020}.
To complement these existing upper bounds, our paper presents new exact formulas for finite-time mean-squared estimation errors of decentralized TD learning with linear function approximation. We adopt the setup in \cite{Doan2021} where the Markov noise is considered  and the projection in TD updates is removed.
We view the decentralized TD learning method as a Markovian jump linear system (MJLS), and apply standard results in the MJLS theory \cite{costa2006} to quantify the finite-time estimation errors
exactly. 
Various implications of our exact formulas on the algorithm performance are also discussed.  
One important finding is that under a necessary and sufficient stability condition, the mean-squared TD estimation error will converge to an exact limit at a specific exponential rate.
We also apply perturbation analysis to characterize how the learning rate choice will affect the algorithm performance.


It is worth mentioning that our work is inspired by a recent line of research
on control-oriented analysis for iterative learning/optimization algorithms \cite{Lessard2014,hu17a,fazlyab2017analysis,hu17b,sundararajan2017robust,hu2018dissipativity,seidman2019control,mohammadi2020robustness,sundararajan2020analysis,hu2021analysis,gannot2021frequency,lee2019unified,Guo2022}, and can be viewed as an extension of  \cite{hu2019characterizing}, which applies the MJLS theory to analyze the centralized TD learning algorithms.

\section{PRELIMINARIES}

\subsection{Notation}
 The set of $n$-dimensional real vectors is denoted as $\R^n$. Let $\bm{1}_n \in \R^n$ be a vector whose elements are all 1. We denote the $n\times n$ identity matrix as $I_n$. The kronecker product of two matrices $A$ and $B$ is denoted as $A \otimes B$. Let $\vect$ denote the standard vectorization operation that stacks the columns of a matrix into a vector. 
Let
$\sym$ denote the symmetrization operation,  
 We use $\diag(H_i)$ to denote a matrix whose $(i,i)$-th block is $H_i$ and all other blocks are zero. 
The spectral radius of a square matrix $H$ is denoted as $\sigma(H)$. Clearly, $H$ is Schur stable if $\sigma(H) < 1$.  The eigenvalue with the largest magnitude of $H$ is denoted as $\lambda_{\max} (H)$. The eigenvalue with the largest real part of $H$ is denoted as $\lambda_{\max \real} (H)$.

\subsection{Multi-agent reinforcement learning}
In this paper, we consider the policy evaluation problem in  multi-agent reinforcement learning. Specifically, $M$ agents will cooperate over a communication network $\mathcal{G}$ to compute the value function for a multi-agent Markov decision process (MDP) in a shared environment.
The multi-agent MDP is described by the following tuple
$$\left( \mathcal{S}, \{ \mathcal{A}_m \}_{m=1}^M, P, \{ R_m\}_{m=1}^M, \gamma, \mathcal{G} \right)$$
where $\mathcal{S}$ is a finite set of global states shared by all the agents,  $\mathcal{A}_m$ is a finite set of actions available to agent $m$, $P$ is the global transition kernel for the shared environment, $R_m$ is the local immediate reward observed by agent $m$, $\gamma$ is the discount factor, and $\mathcal{G}$ is the communication network. At every time step $k$, each agent $m$ will observe the global state $s^k\in \mathcal{S}$ of the shared environment, and then take  an  action $a_m^k\in \mathcal{A}_m$ based on a local policy $\pi_m$.
As a consequence of the joint actions of all the agents, the shared environment will transit to a new state $s^{k+1}\in \mathcal{S}$ . In addition, each agent $m$ will also receive a reward $R_m(s^k,s^{k+1})$ which is only revealed locally.\footnote{At step $k$, the reward $R_m$ will actually depend on $s^k$, $a_m^k$, and $s^{k+1}$. Since the local policy $\pi_m$ does not change over time, we slightly abuse our notation by using $R_m(s^k, s^{k+1})$ to denote the reward under policy $\pi_m$.}
We
emphasize that there is no centralized policy that can access all the action/reward information. 
The agents can only communicate with each other through the network $\mathcal{G} = (\M, \E)$, where $\M:=\{1,2,\ldots,M\}$ is the vertex set, and $\mathcal{E} := \mathcal{V} \times \mathcal{V}$ represents the edge set. Let $\N_m \subset \M$ denote the neighbor(s) of agent $m\in \M$. 

For multi-agent policy evaluation, the agents will cooperate over the network $\mathcal{G}$  to compute the so-called value function which is defined to be the following expected sums of discounted rewards:
\begin{equation}
V_\G(s) = \mathbb{E} \left[ \frac{1}{M} \sum_{m \in \M} \sum_{k=0}^\infty \gamma^k R_m(s^k,s^{k+1}) | s(0) = s \right].
\end{equation}
One can show that the value function $V_\G(s) $ satisfies the following multi-agent Bellman equation:
\begin{equation}
V_\G(s) =  \sum_{s' \in \mathcal{S}} {P}_{ss'}  \left[ \frac{1}{M}\sum_{m \in \M} R_m(s,s') + \gamma V_\G(s')  \right].
\end{equation}
where ${P}_{ss'}$ denotes the transition probability from the current state $s$ to the next state $s'$ under the stationary policies $\{\pi_m\}_{m=1}^M$. For many applications, the transition model is unknown, and the multi-agent Bellman equation cannot be directly solved. Next, we will review the decentralized temporal difference (TD) learning which can be used for model-free policy evaluation.


\subsection{Decentralized TD(0) with linear function approximation}

When the size of the state space $\mathcal{S}$ is very large, exact computation of $V_\G$ for all $s\in \mathcal{S}$ will be intractable. In this paper, the linear function approximation is considered, and the value function will be estimated as $V_\G(s)\approx  {\phi}^\tp(s) \theta$, where $\phi$ is some pre-selected feature vector, and $\theta\in \R^p$ is the weight to be determined.
Then a good estimator for the value function can be obtained by finding the optimal weight $\theta^*$ that minimizes the so-called projected Bellman error.

In the decentralized setting, the reward/action information is kept locally, and the agents have to cooperate over the communication network for finding $\theta^*$.  The idea of decentralized TD learning is that the agents can just share their local TD estimates of $\theta^*$ with their neighbors via the communication network $\mathcal{G}$ and then reach a consensus for a global estimate. The network topology is captured by the weighted adjacency matrix $W$. Let the $mm'$-th entry of $W$ be denoted as $W_{mm'}$. Note that  
$W$ is set to satisfy  $W_{mm'} > 0$ for $m' \in \N_m$, and $W_{mm'} = 0$, otherwise. Then the agents can share their local TD estimates according to $W$.

Now we formalize the decentralized TD(0) method, and a pseudo code is provided as in Algorithm \ref{decentralized_TD}.
 Each agent $m$ updates the local weight $\theta_m^k$ as a estimate of $\theta^*$. At every iteration, each agent $m$ first exchanges its estimation with the neighbors in $\N_m$, and then make the following update:
\begin{align} \label{local_ite}
    \theta_m^{k+1} = \sum_{m'\in \M} W_{mm'} \theta_{m'}^{k} + \alpha  {\phi}(s^k) d^k, 
\end{align}  
where $\alpha$ is the learning rate, $W_{mm'} \in [0,1]$ is the network weight for the edge $(m,m')$, and $d^k$ is given by
\begin{align}
    d^k=(\gamma {\phi}(s^{k+1})-   {\phi}(s^k) )^\tp  \theta_m^{k} + R_m(s^k,s^{k+1}).
\end{align}
The above algorithm combines TD learning with consensus.
It is expected that $\theta_m^k$  will converge to some neighborhood around $\theta^*$ if the learning rate is properly chosen. 

\begin{algorithm}[t!] \label{decentralized_TD}
\SetAlgoLined
\textbf{Input}: $\alpha > 0$, $\phi(s)\,\,  \forall ~ s \in \mathcal{S}$, $W$, $\gamma$  \\
\textbf{Initialization}: $\{\theta_m(0) \}_{m\in \M}$  \\
\textbf{Iteration}:\\
For $k = 0,1,\cdots$, agent $m \in \M$ implements  \\ 
a. Exchange $\theta_m^k$ with agent $m' \in \N_m$ \\
b. Observe $s^k$,$s^{k+1}$, and $R_m(s^k,s^{k+1})$ \\
c. Update the weight:
\begin{align*}
d^k&=(\gamma {\phi}(s^{k+1})-   {\phi}(s^k) )^\tp  \theta_m^{k} + R_m(s^k,s^{k+1}) \\
\theta_m^{k+1} &= \sum_{m'\in \M} W_{mm'} \theta_{m'}^{k} + \alpha  {\phi}(s^k)d^k.
\end{align*} 
 \caption{Decentralized TD(0) Algorithm}
\end{algorithm}

\subsection{Problem statement}
In this paper, we are interested in exact analysis of the finite-time estimation error $\frac{1}{M}\sum_{m=1}^M\mathbb{E}\norm{\theta_m^k-\theta^*}^2$ for the above decentralized TD(0) method. We will present closed-form analytical formulas to quantifying such TD estimation errors and discuss the implications for algorithm performance and design.
Our analysis requires some standard assumptions used in the literature \cite{Sun2020, Doan2021,Zeng2021,Wang2020}. First, we adopt the following assumption on the underlying communication structure.

\begin{assumption} \label{ass_W}
The communication network is connected and undirected. The matrix $W$ is doubly stochastic, i.e., $\sum_{m=1}^M W_{mm'} = 1$ for all $m'$, and $\sum_{m' = 1}^M W_{mm'} = 1$ for all~$m$.
\end{assumption}

Recall that $\theta^*$ is the solution to the projected multi-agent Bellman equation. To ensure the existence and uniqueness of $\theta^*$, the following standard assumption is required.

\begin{assumption} \label{P}
The Markov chain $\{ s^k \}$  is irreducible and aperiodic\footnote{Since the policies $\{\pi_m\}_{m=1}^M$ have been fixed over time, the random process $\{s^k\}$ just becomes a Markov chain}.
All feature vectors  are linearly independent.  
\end{assumption}

\section{Main Analysis Framework via MJLS Theory}

\subsection{Connections between decentralized TD(0) and MJLS}
Markov jump linear systems have been extensively studied in the controls literature \cite{costa2006}. Typically, a MJLS is governed by a state-space model in the following form:
\begin{align}\label{MJLS1}
    \xi^{k+1}=H(z^k) \xi^k+ G(z^k) u^k,
\end{align}
where $\xi^k$ is the state, $u^k$ is the input, and $z^k$ is the so-called jump parameter sampled from a Markov chain. 
In this section, we show that the decentralized TD(0) method \eqref{local_ite} can be viewed as a special case of \eqref{MJLS1} such that existing analysis tools from the MJLS theory \cite{costa2006} can be readily applied. To rewrite \eqref{local_ite} as a MJLS, we can first augment $\bmat{(s^{k+1})^\tp \,\, (s^k)^\tp}^\tp\in \mathcal{S}\oplus\mathcal{S}$ as a new vector~$z^k$. We set $n:=|\mathcal{S}|^2$ , and then there is a one-to-one mapping from $\mathcal{S}\oplus \mathcal{S}$ to the set $\mathcal{N}=\{1,2,\cdots, n\}$. Without loss of generality, $\{z^k\}$ can be set up as a Markov chain sampled from $\mathcal{N}$.
 Given any $z^k$, we define  $A(z^k)$ and $b(z^k)$ as follows:
\begin{align}
    &A(z^k) =  {\phi}(s^k) (\gamma  {\phi}(s^{k+1})-   {\phi}(s^k) )^\tp, \\
    &b_m(z^k) =  R_m(s^k,s^{k+1}) {\phi}(s^k)  .
\end{align}
Therefore, we can rewrite \eqref{local_ite} as 
\begin{equation} 
    \theta_m^{k+1} = \sum_{m'\in \M} W_{mm'} \theta_{m'}^{k} + \alpha \left(A(z^k) \theta_{m}^{k} + b_m(z^k)\right), \label{local_it}
\end{equation} 
Next, we define the following two matrices\footnote{To ease the application of the MJLS theory, our definitions are slightly different from the ones used in \cite{Doan2021,Sun2020}.}:
\begin{align*}
&\Theta := \bmat{  \theta_1  &\theta_2 &\cdots &\theta_M } \in \R^{p\times M}, \\
&B(z^k) := \bmat{  b_1(z^k) & b_2(z^k) &\cdots &b_M(z^k)} \in \R^{p\times M}.
\end{align*}
Then, the update rule \eqref{local_it} can be compactly rewritten as:
\begin{equation} \label{all_agent_update}
\Theta^{k+1} = \alpha A(z^k) \Theta^{k}  + \Theta^{k}  {W}^\tp + \alpha {B}(z^k).
\end{equation}
Now it becomes obvious that we can just vectorize \eqref{all_agent_update} to get a MJLS with $z^k$ being the jump parameter.

To analyze the TD estimation error in \eqref{all_agent_update}, some characterization for $\theta^*$ is needed.
Assumption \ref{P} implies that the Markov chain $\{z^k\}$ admits a unique stationary distribution with only positive entries. In addition, there exists a matrix $\bar{A}$ and vectors $\bar{b}_m$ (for all $m\in \M$) such that:
\begin{align}
   \lim_{k\rightarrow \infty} \mathbb{E}(A(z^k)) = \bar{A} , \,\,\, \lim_{k\rightarrow \infty}\mathbb{E}(b_m(z^k))  = \bar{b}_m.
\end{align}
It can be further shown that all the eigenvalues of $\bar{A}$ have strictly negative real parts. i.e., $\bar{A}$ is Hurwitz \cite{TsiRoy1997}.
Let $\bar{\mathbf{b}} =  \frac{1}{M} \sum_{m=1}^M \bar{b}_m$. Consequently, the optimal weight $\theta^*$ exists and has to be the unique solution to the equation $\bar{A} \theta^* +  \bar{\mathbf{b}}= 0$. See \cite{Sun2020, Doan2021,Zeng2021,Wang2020} for more explanations.
Now we can define:
\begin{equation}
\Theta^* := \bmat{  \theta^*  &\theta^* &\cdots &\theta^* } \in \R^{p\times M}.
\end{equation}
Denoting $\Psi^k=\Theta^{k} - \Theta^*$, we can rewrite \eqref{all_agent_update} as follows:
\begin{equation} \label{all_agent_update2}
\Psi^{k+1}  = \alpha A(z^k) \Psi^{k}  + \Psi^{k}  {W}^\tp + \alpha ( {B}(z^k) + A(z^k)\Theta^*).
\end{equation}
We can vectorize \eqref{all_agent_update2} and obtain 
\begin{align} 
\vect\left(\Psi^{k+1}\right) &= (I_M \otimes (\alpha A(z^k))  +    {W} \otimes I_p ) \vect(\Psi^k)   \nonumber \\
&+ \alpha \vect\left( {B}(z^k)+ A(z^k)\Theta^* \right), \label{all_agent_update3}
\end{align}
which is a special case of the MJLS model \eqref{MJLS1}. If we set $n_\xi=Mp$ and denote $\xi^k=\vect(\Psi^k) \in \R^{n_\xi}$, then \eqref{all_agent_update3} is equivalent to 
\begin{align}
\label{eq:DTD_MJLS}
\xi^{k+1} = H(z^k) \xi^k + G(z^k),
\end{align}
where $H(z^k)\in \R^{n_\xi \times n_\xi}$ and $G(z^k)\in \R^{n_\xi}$ are specified as
\begin{align*}
   H(z^k)  &=  \alpha I_M \otimes A(z^k)   +    {W} \otimes I_p, \\
   G(z^k) &= \alpha \vect\left( {B}(z^k)+ A(z^k)\Theta^* \right).
\end{align*}
Clearly, \eqref{eq:DTD_MJLS} is a special case of \eqref{MJLS1} with $u^k=1$ for all~$k$. At every iteration, the jump parameter $z^k\in \mathcal{N}$ is sampled from the underlying Markov chain. 
When $z^k=i$, we denote $H(z^k)=H_i$ and $G(z^k)=G_i$. Obviously, we have $H(z^k)\in\{H_i\}_{i=1}^n$ and $G(z^k)\in \{G_i\}_{i=1}^n$ for all $k$.

It is straightforward to verify that the mean-squared estimation error for the decentralized TD(0) method satisfies
\begin{align*}
   \frac{1}{M}\sum_{m=1}^M \mathbb{E}\norm{\theta_m^k-\theta^*}^2=\frac{1}{M}\mathbb{E}\norm{\vect(\Psi^k)}^2=\frac{1}{M}\mathbb{E}\norm{\xi^k}^2.
\end{align*}
For convenience, we denote $\delta^k:=\frac{1}{M}\mathbb{E}\norm{\xi^k}^2$.
In the existing literature \cite{Doan2021, Sun2020, Zeng2021}, there are several
upper bounds for $\delta^k$. 
Next, we will show how to apply well-known results from the MJLS theory~\cite{costa2006} to obtain exact formulas for $\delta^k$.



\subsection{Exact formulas for finite-time estimation errors}

Now 
we apply standard MJLS theory \cite[Proposition 3.35]{costa2006} to analyze the decentralized TD learning scheme \eqref{eq:DTD_MJLS}. We will show that the mean and covariance of $\{\xi^k\}$ are governed by a simple LTI system. 

To apply the MJLS theory, 
we need the following notation:
$$q_i^k = \mathbb{E}(\xi^k\textbf{1}_{\{z^k = i\}} ), \qquad Q_i^k = \mathbb{E}(\xi^k (\xi^k)^\tp \textbf{1}_{\{z^k = i\}} ),$$
where $\textbf{1}_{\{z^k = i\}}$ is an indicator function defined as
 $\textbf{1}_{\{z^k = i\}}  = 1$ if $z^k = i$ and $\textbf{1}_{\{z^k = i\}} = 0$ otherwise. 
 Obvious, the mean and covariance of $\xi^k$ can be calculated as
$$\mathbb{E}(\xi^k) = \sum_{i=1}^n q_i^k, \qquad \mathbb{E}(\xi_k \xi_k^\tp) =  \sum_{i=1}^n Q_i^k.$$
Based on standard results in the MJLS theory \cite[Proposition 3.35]{costa2006}, we can calculate $q_j^k$ and $Q_j^k$ iteratively as follows:
\begin{align*}
q_j^{k+1} & = \sum_{i=1}^n p_{ij} (H_i q_i^k +  p_i^k G_i),\\ 
Q_j^{k+1} & = \sum_{i=1}^n p_{ij} (H_i Q_i^k H_i^\tp + 2 \sym(H_i q_i^k G_i^\tp) +  p_i^k G_i G_i^\tp),
\end{align*}
where  $p_{ij}:= \mathbb{P}(z^{k+1} = j | z^k = i)$, and
 $p_i^k:= \mathbb{P}(z^k = i)$. Recall that the mean-squared TD estimation error is defined as $\delta^k=\frac{1}{M}\mathbb{E}\norm{\xi^k}^2$.
Denoting $(q^k)^\tp:= \bmat{(q_1^k)^\tp & \cdots & (q_n^k)^\tp}$ and $\vQ^k:= \vect(\bmat{Q_1^k, \cdots, Q_n^k})$, and 
 we can just vectorize the above recursion and obtain the following simple LTI system:
 \begin{align} \label{eq:lti_key}
\bmat{q^{k+1} \\  \vQ^{k+1}} &= \bmat{\mathcal{H}_{11} &0 \\ \mathcal{H}_{21} &\mathcal{H}_{22}} \bmat{q^{k} \\  \vQ^{k}} +\bmat{u_q^k \\ u_Q^k},\\
\delta^k& = C_\delta \vQ^k,\label{eq:deltafor}
\end{align}
where $\mathcal{H}_{11}$, $\mathcal{H}_{21}$, $\mathcal{H}_{22}$, $C_\delta$, $u_q^k$, and $u_Q^k$ are given by
\begin{align*}
\begin{split}
\mathcal{H}_{11}&= \bmat{p_{11} H_1 & \ldots & p_{n1} H_n\\ \vdots & \ddots & \vdots\\ p_{1n} H_1 & \ldots & p_{nn} H_n},\\
\mathcal{H}_{22}&= \bmat{p_{11} H_1\otimes H_1 & \ldots & p_{n1} H_n\otimes H_n\\ \vdots & \ddots & \vdots\\ p_{1n} H_1\otimes H_1 & \ldots & p_{nn} H_n\otimes H_n},\\
\mathcal{H}_{21}&=\bmat{p_{11} S_1  & \ldots &  p_{n1} S_n,  \\ \vdots & \ddots & \vdots \\  p_{1n} S_1 &  \ldots & p_{nn} S_n },\\
C_\delta&=\frac{1}{M} (\bm{1}_n^\tp \otimes \vect(I_{n_\xi})^\tp),\\
u_q^k&= \bmat{p_{11} G_1  & \ldots &  p_{n1} G_n  \\ \vdots & \ddots & \vdots \\  p_{1n} G_1 &  \ldots & p_{nn} G_n }\bmat{p_1^k I_{n_\xi}\\ \vdots \\ p_n^k I_{n_\xi}},\\
u_Q^k&= \bmat{p_{11} G_1\otimes G_1  & \ldots &  p_{n1} G_n\otimes G_n  \\ \vdots & \ddots & \vdots \\  p_{1n} G_1\otimes G_1 &  \ldots & p_{nn} G_n\otimes G_n }\bmat{p_1^k I_{n_\xi^2}\\ \vdots \\ p_n^k I_{n_\xi^2}}.
\end{split}
 \end{align*}
Notice that the term $S_i$ is defined as $S_i=H_i\otimes G_i+G_i\otimes H_i$ for all $i\in \mathcal{N}$.
The LTI system representation \eqref{eq:lti_key} is quite standard for MJLS models \cite{costa2006,hu2019characterizing}.
 Based on \eqref{eq:lti_key},
the mean and covariance of $\{\xi^k\}$ can be exactly calculated as
\begin{align}
    q^k&=(\mathcal{H}_{11})^k q^0+\sum_{t=0}^{k-1} (\mathcal{H}_{11})^{k-1-t}u_q^t,\\
    \hat{\textbf{Q}}^k&=(\mathcal{H}_{22})^k \hat{\textbf{Q}}^0+\sum_{t=0}^{k-1} (\mathcal{H}_{22})^{k-1-t}(\mathcal{H}_{21} q^t+u_Q^t).\label{eq:Qfor}
\end{align}
This directly leads to the following result.
\begin{theorem}
The finite-time estimation error of decentralized TD(0) can be calculated as
\begin{align*}
  \delta^k=C_\delta(\mathcal{H}_{22})^k \hat{\textbf{Q}}^0+\sum_{t=0}^{k-1} C_\delta (\mathcal{H}_{22})^{k-1-t}(\mathcal{H}_{21} q^t+u_Q^t).
\end{align*}
\end{theorem}
\begin{proof}
Combining \eqref{eq:Qfor} with \eqref{eq:deltafor} immediately leads to the desired conclusion.
\end{proof} 

Our formulas have several important implications which will be discussed later.
\begin{remark} \label{two_bounding}
Previous work on finite time analysis of decentralized TD(0) relied on the following decomposition~\cite{Doan2021}:
\begin{equation}
\theta_m^k - \theta^* = \underbrace{(\theta_m^k - \bar{\theta}^k )}_{\text{``consensus  error"}}+   \underbrace{(\bar{\theta}^k  - \theta^*)}_{\text{``optimality error"}}, 
\end{equation}
where $\bar{\theta}^k = \frac{1}{M}\sum_{m= 1}^M \theta_m^k$ is the average of the local TD estimates from all agents. Since $W$ is doubly stochastic, averaging \eqref{local_ite} over all $m$ leads to
$\bar{\theta}^{k+1} = \bar{\theta}^k + \alpha \left(A(z^k)  \bar{\theta}^{k} + \bar{b}(z^k) \right)$,
where $\bar{b}(z^k) = \frac{1}{M}\sum_{m=1}^M  b_m(z^k).$ It is obvious that the iterative process of $\{\bar{\theta}^k\}$ reduces to the ``single-agent" TD(0) scheme, whose finite-time behaviors have been well understood~\cite{srikant2019finite}. Existing work addressed the consensus error term separately, and various upper bounds for the mean-squared TD estimation errors have been obtained \cite{Doan2021, Sun2020, Zeng2021}. 
Using our MJLS approach, such a decomposition is not needed, and exact formulas for the TD estimation errors are obtained.
\end{remark}

\subsection{Implications for algorithm performance}

Now we discuss some implications of our exact formulas.

$\bullet$ \textbf{Stability}: The LTI system \eqref{eq:lti_key} is stable if and only if $\mathcal{H}_{22}$ is Schur stable.\footnote{By Proposition 3.6 in \cite{costa2006}, $\mathcal{H}_{11}$ is Schur stable  if $\mathcal{H}_{22}$ is Schur stable. Hence the stability of \eqref{eq:lti_key} is completely determined by $\sigma(\mathcal{H}_{22})$.} Notice that $\mathcal{H}_{22}$ depends on $W$ and $\alpha$. In the next section, we will show that we can choose sufficiently small $\alpha$ to achieve $\sigma(\mathcal{H}_{22})<1$ and ensure the stability of~\eqref{eq:lti_key}.

$\bullet$ \textbf{Steady-state estimation error}:
If $\sigma(\mathcal{H}_{22})<1$, then the system \eqref{eq:lti_key} is stable and the estimation error $\delta^k$ is guaranteed to converge to a stationary value. 
 To see this, notice that the Markov chain $\{z^k\}$ will converge to a stationary distribution geometrically fast under  Assumption~\ref{P}. Denote  $p^k := \bmat{p_1^k &p_2^k &\cdots &p_n^k}^\tp$ and $p^\infty:=\lim_{k\rightarrow \infty} p^k$.  Then the limits of $u_q^k$ and $u_Q^k$ also exist. We denote $u_q^\infty:=\lim_{k\rightarrow\infty} u_q^k$ and $u_Q^\infty:=\lim_{k\rightarrow\infty} u_Q^k$. We have
\begin{align*}
u_q^\infty&= \bmat{p_{11} G_1  & \ldots &  p_{n1} G_n  \\ \vdots & \ddots & \vdots \\  p_{1n} G_1 &  \ldots & p_{nn} G_n }\bmat{p_1^\infty I_{n_\xi}\\ \vdots \\ p_n^\infty I_{n_\xi}},\\
u_Q^\infty&= \bmat{p_{11} G_1\otimes G_1  & \ldots &  p_{n1} G_n\otimes G_n  \\ \vdots & \ddots & \vdots \\  p_{1n} G_1\otimes G_1 &  \ldots & p_{nn} G_n\otimes G_n }\bmat{p_1^\infty I_{n_\xi^2}\\ \vdots \\ p_n^\infty I_{n_\xi^2}}.
 \end{align*}
If  $\sigma(\mathcal{H}_{22})<1$, the system \eqref{eq:lti_key} is stable. Based on standard LTI results (e.g. Proposition 3 in  \cite{hu2019characterizing}), $(q^k, \vQ^k, \delta^k)$ will converge to some exact limit values which are given as
\begin{align*}
    q^\infty&=\lim_{k\rightarrow\infty} q^k=(I-\mathcal{H}_{11})^{-1} u_q^\infty,\\
    \vQ^\infty&=\lim_{k\rightarrow\infty} q^k=(I_{nn_\xi^2}-\mathcal{H}_{22})^{-1}(\mathcal{H}_{21}q^\infty+ u_Q^\infty),\\
    \delta^\infty&=\lim_{k\rightarrow\infty} \delta^k=C_\delta(I_{nn_\xi^2}-\mathcal{H}_{22})^{-1}(\mathcal{H}_{21}q^\infty+ u_Q^\infty).
\end{align*}
Our analysis characterizes the exact limit of $\delta^k$, while 
the existing results from \cite{Doan2021, Sun2020,Zeng2021} lead to various upper bounds on $\limsup_{k\rightarrow\infty}\delta^k$. 
Notice $q^\infty\neq 0$ in general.
In the next section, we will show $q^\infty=O(\alpha)$, $\vQ^\infty=O(\alpha)$, and $\delta^\infty=O(\alpha)$ for small $\alpha$ if Assumptions \ref{ass_W} and \ref{P} are given.


$\bullet$ \textbf{Convergence rate}: The convergence rate of $\delta^k$ can also be characterized using standard LTI theory. 
Based on Assumption \ref{P}, we have $\| p^k - p^\infty  \| \le c \tilde{\rho}^k$ for some $c$ and $0<\tilde{\rho} < 1$. Here $\tilde{\rho}$ is the mixing rate of $\{z^k\}$. 
A direct application of \cite[Proposition~3]{hu2019characterizing} leads to the following estimation error bound:
\begin{align} 
  \delta^\infty - C_1 \rho^k \le  &\delta^k \le \delta^\infty + C_1 \rho^k, \label{eq:bounds2}
\end{align}
where  
$\rho := \max\{ \sigma(\mathcal{H}_{11}) + \varepsilon, \sigma(\mathcal{H}_{22}) + \varepsilon, \tilde{\rho} \} < 1$ captures the convergence rate, and $C_1$ is some constant.
Here $\varepsilon$ can be any arbitrarily small positive number.  Clearly, the convergence rate $\rho$ depends on $\sigma(\mathcal{H}_{11})$, $\sigma(\mathcal{H}_{22})$, and $\tilde{\rho}$. When $\tilde{\rho}$ is the dominating rate, increasing $\alpha$ may not improve the convergence speed. However, $\sigma(\mathcal{H}_{11})$ will eventually becomes the dominating term when $\alpha$ is small enough. It is also worth mentioning that $\sigma(\mathcal{H}_{11})$ and $\sigma(\mathcal{H}_{22})$ depend on $W$. This dependence characterizes how the network topology  will affect the convergence rate of the decentralized TD(0) method. More discussions on the dependence of $\rho$ on $\alpha$ will be given in the next section.

\section{Discussions on Learning Rate Tuning} 
\label{per_markov}


In this section, we will show that the following results hold for  small $\alpha$:
\begin{align}
\label{eq:approx1}
    &\sigma(\mathcal{H}_{22})=  1 +   2\real(\lambda_{\max \real} (\bar{A})) \alpha+ o(\alpha)<1,\\
    \label{eq:approx2}
    &\sigma(\mathcal{H}_{11})=  1 +   \real(\lambda_{\max \real} (\bar{A})) \alpha+ o(\alpha)<1, \\
    \label{eq:approx3}
    &\delta^\infty=O(\alpha).
\end{align}
Based on such perturbation analysis results, it is expected that one can decrease the learning rate $\alpha$ to stabilize the learning process and obtain a smaller steady-state estimation error $\delta^\infty$. However, decreasing $\alpha$ leads to a larger value of $\sigma(\mathcal{H}_{11})$, meaning that the convergence is slowed down. Such design trade-off is consistent with the upper bounds for $\delta^k$ in the existing literature.

The analysis in this section relies on the perturbation theory. For simplicity, we denote $A(z^k)=A_i$ and $B(z^k)=B_i$ when $z^k=i\in \mathcal{N}$. We also denote the transition matrix of $\{z^k\}$ as $P_z$. Hence the $(i,j)$-th entry of $P_z$ is equal to~$p_{ij}$.


\subsection{Eigenvalue perturbation analysis}

To show \eqref{eq:approx1} and \eqref{eq:approx2},
we will perform
eigenvalue perturbation analysis.  The following fact is useful.
\begin{fact}
Suppose $\lambda$ is a semisimple eigenvalue of $K_0$ with multiplicity $r$. 
Suppose $Y = \bmat{y_1^\tp & \cdots  & y_r^\tp}^\tp$ and 
$X = \bmat{x_1 &\cdots &x_r}$, where
$(y_1,\cdots,y_r)$ and 
$(x_1,\cdots,x_r)$ are chosen to be independent left and right eigenvectors of $K_0$ associated with eigenvalue $\lambda$ and satisfy $Y X = I_r$.
Then there are $r$ eigenvalues for the perturbed matrix $K_0+\alpha K_1$ yielding the first-order expansion $\lambda + \eta \alpha + o(\alpha)$ for small $\alpha$, where $\eta$ is an eigenvalue of the $r\times r$ matrix $Y K_1 X$. 
\end{fact}

\vspace{0.05in}
Now we apply the above well-known fact\footnote{See the remark placed behind \cite[Theorem 2.1]{moro1997lidskii} for more explanations.} to analyze $\sigma(\mathcal{H}_{11})$ and $\sigma(\mathcal{H}_{22})$.

\noindent$\bullet$ \textbf{Analysis for $\sigma(\mathcal{H}_{11})$}: 
Let us specify $K_0$ and $K_1$ as
$$
K_0=P_z^\tp\otimes W\otimes I_p, \,\,K_1=(P_z^\tp\otimes I_{n_\xi})\diag(I_M\otimes A_i).
$$
Then we have $\mathcal{H}_{11}=K_0+\alpha K_1$.
From Assumptions \ref{ass_W} \& \ref{P}, we know that $\lambda_{\max}(K_0)=1$ is a semisimple eigenvalue of $K_0$ with multiplicity $p$. After examining the eigenvectors associated with $\lambda_{\max}(K_0)$, we choose $Y=\frac{1}{M}\mathbf{1}_n^\tp\otimes \mathbf{1}_M^\tp\otimes I_p$ and $X=p^\infty\otimes \textbf{1}_M\otimes I_p$ such that $YX=I_p$. We can verify 
$$
YK_1X=\frac{1}{M}\sum_{i=1}^n p_i^\infty(\mathbf{1}_M^\tp\otimes I_p)(I_M\otimes A_i)(\mathbf{1}_M\otimes I_p).
$$
After simplification, we get $YK_1 X=\sum_{i=1}^n p_i^\infty A_i=\bar{A}$. Therefore, we can obtain the following result:
\begin{align*}
\lambda_{\max} (\mathcal{H}_{11}) \approx  1 +   \lambda_{\max \real} (\bar{A}) \alpha+ o(\alpha),
\end{align*}
which directly leads to the perturbation formula \eqref{eq:approx1}.

\noindent$\bullet$ \textbf{Analysis for $\sigma(\mathcal{H}_{22})$}: To prove \eqref{eq:approx2}, we can just choose $K_0=P_z^\tp\otimes (W\otimes I_p)\otimes (W\otimes I_p)$ and set $K_1$ to be equal to the following matrix 
\begin{align*}
(P_z^\tp\otimes I_{n_\xi^2})\diag(I_M\otimes A_i \otimes W\otimes I_p+W\otimes I_p \otimes I_M\otimes A_i).
\end{align*}
Then we have $\mathcal{H}_{22}=K_0+\alpha K_1+O(\alpha^2)$. Under mild technical conditions, we can drop the second-order term~$O(\alpha^2)$. Based on Assumption \ref{ass_W} \& \ref{P}, we know $\lambda_{\max}(K_0)=1$ is a semisimple eigenvalue of $K_0$ with multiplicity $p^2$. We can choose $Y$ and $X$ as
\begin{align}
\begin{split}
\label{eq:XY2}
Y&=\frac{1}{M^2}\mathbf{1}_n^\tp\otimes \mathbf{1}_M^\tp\otimes I_p\otimes \mathbf{1}_M^\tp\otimes I_p,\\
X&=p^\infty\otimes \textbf{1}_M\otimes I_p\otimes \textbf{1}_M\otimes I_p.
\end{split}
\end{align}
Obviously, we have $YX=I_{p^2}$. It is also straightforward to verify $Y K_1 X=\bar{A}\otimes I_p+I_p\otimes \bar{A}$. Therefore, we have 
\begin{align*}
\lambda_{\max} (\mathcal{H}_{22}) \approx  1 +   2\lambda_{\max \real} (\bar{A}) \alpha+ o(\alpha),
\end{align*}
which leads to the perturbation result \eqref{eq:approx2}.





\subsection{Steady-state estimation error analysis}
To show \eqref{eq:approx3}, we will use the Laurent expansion of matrix inverse. Our analysis is formalized as follows.

\begin{corollary}
Under Assumptions \ref{ass_W} \& \ref{P},  the following result holds for sufficient small $\alpha$:
  \begin{align*}
  q^\infty=O(\alpha), \,\, \vQ^\infty=O(\alpha), \,\,\mbox{and}\,\, \delta^\infty=O(\alpha).
  \end{align*}
\end{corollary}
\begin{proof}
We will use the following fact which can be viewed as a special case of \cite[Theorem 2.9]{Avrachenkov2013}.
\begin{fact}
Given a singular matrix $D_0$. let $U$ be a matrix whose columns form a basis of the null space of $D_0$. In addition,  let $V$ be a matrix whose columns form a basis for the null space of $D_0^\tp$. Suppose  the perturbed matrix $D_0+\alpha D_1$ is nonsingular for small $\alpha$.
If $V^\tp D_1 U$ is nonsingular, then $(D_0+\alpha D_1)^{-1}$  satisfies the first-order Laurent expansion 
$(D_0 + \alpha D_1)^{-1} = \frac{1}{\alpha} U(V^\tp D_1 U)^{-1}V^\tp +O(1)$.
\end{fact}
\vspace{0.05in}

First, we apply the Laurent expansion approach to analyze $q^\infty=(I-\mathcal{H}_{11})^{-1} u_q^\infty$.
 In this case, 
we choose $D_0$ and $D_1$~as
\begin{align*}
    D_0&=I_{nn_{\xi}}-P_z^\tp \otimes W\otimes I_p,\\
    D_1&=-(P_z^\tp\otimes I_{n_\xi})\diag(I_M\otimes A_i).
\end{align*}
Under Assumptions \ref{ass_W} \& \ref{P}, the null space of $D_0$ is the same as the eigenspace of $P_z^\tp \otimes W\otimes I_p$ for the eigenvalue $1$. Hence we choose $U=p^\infty\otimes \textbf{1}_M\otimes I_p$. Similarly, the null space of $D_0^\tp$ is characterized by $V=\frac{1}{M}\mathbf{1}_n\otimes \mathbf{1}_M\otimes I_p$
Then we have $V^\tp D_1 U=-\bar{A}$, which is nonsingular. Therefore, we have
\begin{align*}
    (I-\mathcal{H}_{11})^{-1}=\frac{1}{\alpha} U \bar{A}^{-1} V^\tp+O(1),
\end{align*}
Notice $G_i=O(\alpha)$ for all $i\in\mathcal{N}$. Hence we have $u_q^\infty=O(\alpha)$. This leads to the following result:
\begin{align*}
   q^\infty= (I-\mathcal{H}_{11})^{-1}u_q^\infty=\frac{1}{\alpha} U \bar{A}^{-1} V^\tp u_q^\infty+O(\alpha).
\end{align*}
Due to the fact that $\bar{A} \theta^* +  \bar{\mathbf{b}}= 0$, it is straightforward to verify $\frac{1}{\alpha} U \bar{A}^{-1} V^\tp u_q^\infty=0$. Hence we have $q^\infty=O(\alpha)$.

The Laurent expansion for $(I-\mathcal{H}_{22})^{-1}$ can be done in a similar way. We can choose $U=X$ and $V=Y^\tp$ where $(X,Y)$ is given by \eqref{eq:XY2}. Then it is not difficult to verify $\vQ^\infty=O(\alpha)$. Finally, we have $\delta^\infty=C_\delta \vQ^\infty=O(\alpha)$. This completes the proof.
\end{proof}
\begin{remark}
The connectedness of the underlying network is essential for our perturbation analysis. Clearly, the choices of $(U,V)$ (for the steady-state error analysis) or $(Y,X)$ (for the eigenvalue perturbation analysis) rely on the connectedness of $W$. However, our analysis does not make it explicit how the spectral gap of $W$ will affect the convergence rate. How to interpret our exact formula for $\delta^k$ in the large learning rate regime is not fully clear at this moment.
It may be interesting to investigate whether $\sigma(\mathcal{H}_{11})$ and $\sigma(\mathcal{H}_{22})$ yield simple upper bounds which have a more explicit  dependence on the spectral gap of $W$. That can potentially lead to some estimation error bounds which are easier to interpret and more consistent with the results in \cite{Doan2021}.
\end{remark}

\section{CONCLUSION}
In this paper, we applied the MJLS theory to study decentralized TD learning with linear function approximation. We present exact formulas for the mean-squared estimation errors of the decentralized TD(0) method, and discuss several implications on the algorithm behaviors.


\bibliographystyle{IEEEtran}
\bibliography{IEEEabrv,my_references}

\end{document}